\documentclass{article}

\usepackage{microtype}
\usepackage{graphicx}
\usepackage{subcaption}
\usepackage{booktabs}
\usepackage{amsmath}
\usepackage{amssymb}
\usepackage{mathtools}
\usepackage{algorithm}
\usepackage{algorithmic}
\usepackage{amsthm}
\usepackage{url}
\usepackage{xcolor}
\usepackage{hyperref}
\usepackage[capitalize,noabbrev]{cleveref}

\def\isaccepted{1}
\IfFileExists{icml2026_weightsymmetry.sty}{
  \usepackage[accepted]{icml2026_weightsymmetry}
}{
  \usepackage[accepted]{icml2026}
}
\makeatletter
\renewcommand{\Notice@String}{\textit{Accepted at the ICML 2026 Workshop on Weight-Space Symmetries: from Foundations to Practical Applications.}}
\makeatother
\definecolor{ruleNavy}{RGB}{8,35,81}

\theoremstyle{plain}
\newtheorem{mpproposition}{Proposition}
\newtheorem{mpcorollary}{Corollary}
\crefname{mpproposition}{Proposition}{Propositions}
\Crefname{mpproposition}{Proposition}{Propositions}
\crefname{mpcorollary}{Corollary}{Corollaries}
\Crefname{mpcorollary}{Corollary}{Corollaries}

\icmltitlerunning{Access Sets Matter: Budgeting Expert Reads for Scalable Weight-Space Model Merging}

\begin{document}

\twocolumn[
\icmltitle{Access Sets Matter: Budgeting Expert Reads for Scalable \\ Weight-Space Model Merging}

\begin{icmlauthorlist}
  \icmlauthor{Yuanyi Wang}{polyu}
  \icmlauthor{Yanggan Gu}{polyu}
  \icmlauthor{Su Lu}{polyu}
  \icmlauthor{Yifan Yang}{polyu} \\
  \icmlauthor{Zhaoyi Yan}{infix}
  \icmlauthor{Congkai Xie}{infix}
  \icmlauthor{Jianmin Wu}{polyu}
  \icmlauthor{Hongxia Yang}{polyu,infix,dayabay}
\end{icmlauthorlist}
\icmlaffiliation{polyu}{The Hong Kong Polytechnic University, PolyU}
\icmlaffiliation{infix}{InfiX.ai}
\icmlaffiliation{dayabay}{Hong Kong Polytechnic University Daya Bay Technology and Innovation Research Institute}
\icmlcorrespondingauthor{Yuanyi Wang}{wangyuanyi713@gmail.com}
\icmlcorrespondingauthor{Hongxia Yang}{hongxia.yang@polyu.edu.hk}
\icmlkeywords{model merging, LLM systems, parameter-efficient adaptation}

\begin{center}
\textbf{Code:} \textcolor{ruleNavy}{\href{https://github.com/wyy-code/mergepipe}{\texttt{wyy-code/mergepipe}}}
\end{center}

\vskip 0.3in
]

\printAffiliationsAndNotice{Corresponding author: Hongxia Yang.}

\begin{abstract}
Weight-space model merging is usually formulated as an algebraic operation on checkpoints, yet at LLM scale the limiting resource is often the set of expert weights that must be read.
We introduce \textbf{MergePipe}, a budget-aware execution layer that casts LLM merging as an \emph{expert access-set} problem: given a merge operator and a checkpoint family in a shared weight coordinate system, choose which expert delta blocks to access under an explicit I/O budget. 
MergePipe indexes parameter blocks, builds deterministic access plans, and executes the induced budgeted merge with replayable manifests. The plan is budget-sound by construction and recovers the full-read merge at full budget; for fixed-coefficient additive operators, the omitted-update error is bounded by the norm of omitted deltas. 
Across Qwen and Llama merging workloads, MergePipe reduces expert-read I/O by up to an order of magnitude and achieves up to $11\times$ speedups. 
Representative budget sweeps show $O(10^{-3})$ parameter deviation from full-read merges and no monotonic degradation on downstream benchmarks.
\end{abstract}

\section{Introduction}
\label{sec:intro}
xxx
Modern LLM development increasingly produces \emph{checkpoint families}: a shared base model, instruction-tuned variants, domain experts, and adapter or delta updates \citep{wortsman2022model,yadav2023ties}. 
These families are becoming weight-space datasets, and model merging offers a post-training primitive for consolidating them into one deployable model without ensembling or another full training run \citep{yang2024model,lu2024merge,wang2025model}.

Most merging work asks how task vectors, signs, sparsity masks, or low-rank updates should be combined. 
For LLM-scale checkpoint families, an equally important execution question appears: \emph{which expert parameters must be read to realize a merge?} 
Naive merge scripts treat checkpoints as opaque files, scan all expert parameters, apply AVG \cite{wortsman2022model}, TIES \cite{yadav2023ties}, DARE-like rules \cite{yu2024language}, and write the output. As the expert pool grows, this expert-read term scales nearly linearly in the number of experts $K$, making iterative merging I/O-bound rather than compute-bound.

MergePipe assumes the standard homologous setting where experts are fine-tuned from a shared base, or have already been aligned into a common weight coordinate system.
It is therefore complementary to symmetry- or permutation-alignment methods: after weights live in one coordinate chart, MergePipe asks which expert deltas should be accessed under a finite budget.

Our thesis is that LLM-scale merging needs a \emph{weight-access abstraction}, not only better merge rules. 
MergePipe treats expert deltas as a budgeted resource and decouples \emph{which} blocks are accessed from \emph{how} accessed deltas are combined. 
Full-read merging fixes access mask $A=\mathbf{1}$; budgeted merging chooses an access mask $A$ under cost $C_{\mathrm{expert}}(A)\le B$ and executes the induced mask-aware operator $\Psi_{\mathrm{op}}$. 
Omitted entries are encoded by the mask and do not trigger storage reads. 
Thus, full-budget execution recovers the standard merge, while lower budgets define an explicit approximation and expose a speed--fidelity frontier.

Our contributions are as follows:
\textbf{(i)} We introduce expert access sets as a budgeted object for weight-space model merging.
\textbf{(ii)} We prove budget soundness, full-budget consistency, and an omitted-update bound for additive merges.
\textbf{(iii)} We instantiate this abstraction in MergePipe, achieving up to order-of-magnitude expert-I/O reduction and $11\times$ speedups on Qwen and Llama checkpoint families, with representative budget sweeps preserving downstream behavior.

\section{Budgeted Access Sets}
\label{sec:cost}

Consider merging a base model $M_0$ with experts $\{M_i\}_{i=1}^{K}$. Let $\mathcal{T}$ be the tensor set and $\mathcal{B}_t$ the deterministic blocks of tensor $t$. For block $(t,b)$, let $M_0[t,b]$ be the base block and $\Delta_{i,t,b}=M_i[t,b]-M_0[t,b]$ expert $i$'s delta. A full-read merge computes
\begin{equation}
  M_{\mathrm{full}}[t,b]=M_0[t,b]+
  \Phi_{\mathrm{op},t,b}\!\left(
  \Delta_{1,t,b},\ldots,\Delta_{K,t,b};\omega\right),
  \label{eq:local-merge}
\end{equation}
where $\Phi_{\mathrm{op}}$ is AVG, TIES, DARE, or another weight-space operator;
$\omega$ denotes the mask for randomized operators.

MergePipe exposes the hidden execution choice with an access mask
\begin{equation}
  A\in\{0,1\}^{|\mathcal{U}|},\quad
  A_{i,t,b}=1\Leftrightarrow
  \text{read expert }i\text{ for }(t,b),
\end{equation}
where $\mathcal{U}=\{(i,t,b):i\in[K],t\in\mathcal{T},b\in\mathcal{B}_t\}$. Each unit has physical read cost $c_{i,t,b}\ge0$, measured in the executor's accounting unit. The controllable expert-read cost is
\begin{equation}
  C_{\mathrm{expert}}(A)=
  \sum_{(i,t,b)\in\mathcal{U}}A_{i,t,b}c_{i,t,b},
  \label{eq:mask-cost}
\end{equation}
while base reads and output writes are checkpoint-boundary costs. The budgeted merge is
\begin{equation}
  M_A[t,b]=M_0[t,b]+
  \Psi_{\mathrm{op},t,b}
  \left(A_{\cdot,t,b},\Delta_{\cdot,t,b};\omega\right),
  \label{eq:budgeted-merge}
\end{equation}
with full-budget consistency
\begin{equation}
  \Psi_{\mathrm{op},t,b}
  (\mathbf{1},\Delta_{\cdot,t,b};\omega)
  =
  \Phi_{\mathrm{op},t,b}
  (\Delta_{1,t,b},\ldots,\Delta_{K,t,b};\omega).
  \label{eq:psi-consistency}
\end{equation}
We require $\Psi_{\mathrm{op}}$ to be \emph{non-anticipatory}: if two delta tuples agree on all selected entries $\{i:A_{i,t,b}=1\}$, they produce the same budgeted output. Hence omitted entries are represented only by the mask, not by reading their contents. Offline \textsc{Analyze} reads used to build sketches or norms are amortized catalog construction; if performed inside a merge run, they are counted in $C_{\mathrm{expert}}^{\mathrm{run}}$.

Planning can be written as the idealized access-set objective
\begin{equation}
  A^\star=\arg\max_{A}\sum_{i,t,b}A_{i,t,b}s(i,t,b)
  \quad\mathrm{s.t.}\quad C_{\mathrm{expert}}(A)\le B ,
  \label{eq:access-objective}
\end{equation}
where $s(i,t,b)$ comes from norms, sketches, coverage, or fallback metadata. MergePipe implements this objective with deterministic greedy or score-per-byte heuristics rather than exact knapsack optimization.

\begin{mpproposition}[Budgeted execution invariant]
\label{prop:budget-soundness}
Let $A$ be an access mask with $C_{\mathrm{expert}}(A)\le B$.
Assume the planner and executor use the same nonnegative read costs $c_{i,t,b}$, and the charged execution trace reads only selected expert blocks, i.e., $N^{\mathrm{run}}_{i,t,b}\le A_{i,t,b}$.
Then
\begin{equation}
C_{\mathrm{expert}}^{\mathrm{run}}(A)
=\sum_{i,t,b}N^{\mathrm{run}}_{i,t,b}c_{i,t,b}
\le C_{\mathrm{expert}}(A)\le B .
\end{equation}
Moreover, if $A=\mathbf{1}$, then the budgeted operator recovers the full-read merge, $M_A=M_{\mathrm{full}}$; for randomized operators, this equality is pathwise under the same seed or mask $\omega$.
\end{mpproposition}

\begin{mpcorollary}[Expert-read fraction under fixed absolute budget]
\label{cor:expert-scaling}
If full-read merging reads all $K$ expert checkpoints with average expert-read cost $\bar{C}>0$, then $C_{\mathrm{expert}}^{\mathrm{full}}(K)=K\bar{C}$.
Under a fixed absolute expert-read budget $B$,
\begin{equation}
\frac{C_{\mathrm{expert}}^{\mathrm{run}}(A)}
{C_{\mathrm{expert}}^{\mathrm{full}}(K)}
\le \frac{B}{K\bar{C}} .
\end{equation}
Thus, when $\bar{C}$ is bounded below and $B$ does not scale with $K$, the expert-read fraction decreases as $O(1/K)$.
\end{mpcorollary}

\begin{mpproposition}[Additive omission bound]
\label{prop:additive-bound}
Assume blocks form a disjoint partition of the parameter vector.
Suppose the budgeted additive operator uses fixed coefficients independent of $A$:
\begin{equation}
\Psi_{\mathrm{add},t,b}(A_{\cdot,t,b},\Delta_{\cdot,t,b})
\!=
\sum_{i=1}^{K}\alpha_{i,t,b}A_{i,t,b}\Delta_{i,t,b}.
\end{equation}
Then
\begin{equation}
M_{\mathrm{full}}[t,b]-M_A[t,b]
\!=
\sum_{i=1}^{K}(1-A_{i,t,b})\alpha_{i,t,b}\Delta_{i,t,b}.
\end{equation}
Let $q_{t,b}(A)=\sum_i(1-A_{i,t,b})|\alpha_{i,t,b}|\|\Delta_{i,t,b}\|_2$. Then
\begin{equation}
\|M_{\mathrm{full}}-M_A\|_2
\le
\left[\sum_{t,b}q_{t,b}(A)^2\right]^{1/2}
\le
\sum_{t,b}q_{t,b}(A).
\label{eq:omission-bound}
\end{equation}
where $M_{\mathrm{full}}$ denotes the same additive operator with $A=1$.
\end{mpproposition}

\Cref{prop:additive-bound} gives a simple justification for norm- or sketch-based access scores in fixed-coefficient additive merges.
With $u_{i,t,b}=|\alpha_{i,t,b}|\|\Delta_{i,t,b}\|_2$, minimizing the looser $\ell_1$ omission bound is equivalent to maximizing retained utility $\sum_{i,t,b}A_{i,t,b}u_{i,t,b}$ under the expert-read budget.
This fidelity bound does not apply to selected-only renormalization, TIES, or DARE, where access masks can change coefficients, random drops, sparsification, or sign election; for these operators, MergePipe provides the same budgeted execution abstraction, while fidelity is evaluated empirically.

\textbf{Remark:} Proofs are provided in Appendix~\ref{app:proofs}.

\section{MergePipe}
\label{sec:method}

\begin{figure*}[t]
  \centering
  \includegraphics[width=0.92\textwidth]{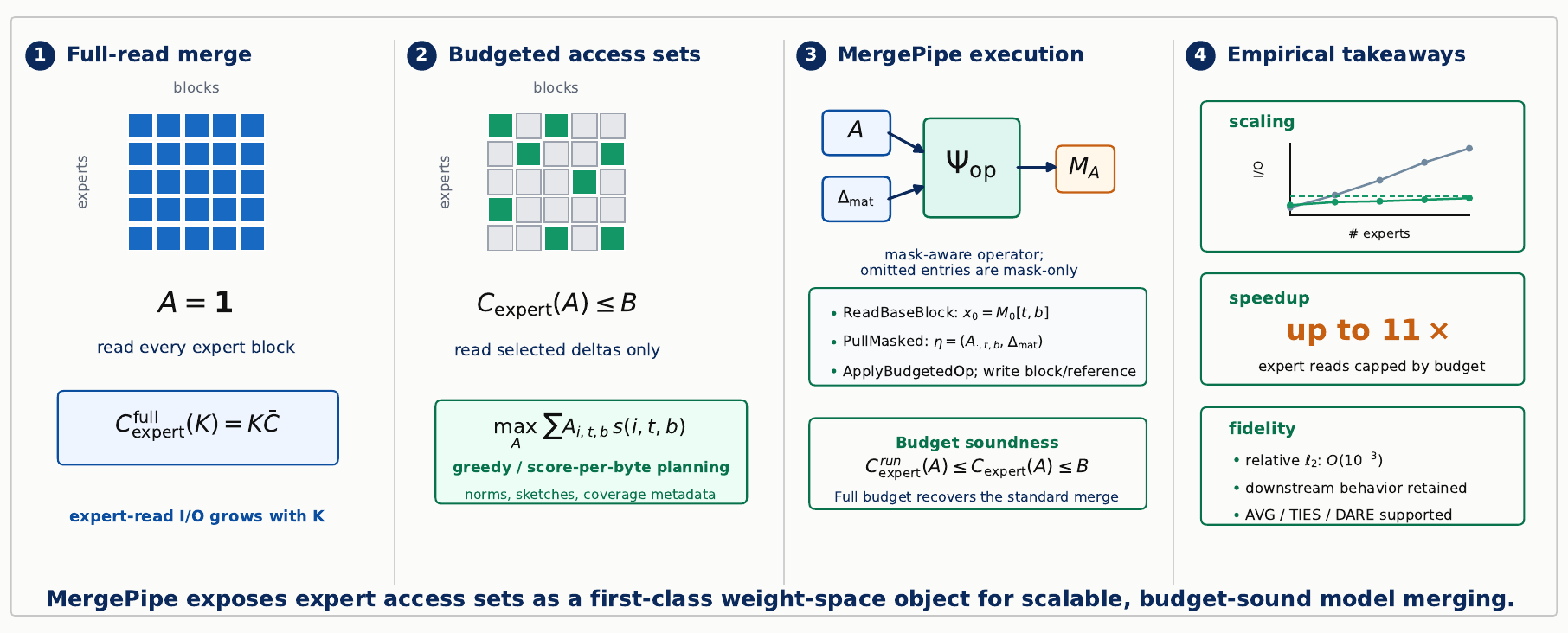}
  \caption{\textbf{Budgeted access sets in weight space.}
  Full-read merging fixes $A=\mathbf{1}$.
  MergePipe chooses a budget-feasible access mask $A$ and executes the induced mask-aware operator $\Psi_{\mathrm{op}}$; omitted entries are represented by the mask and do not trigger expert reads.}
  \label{fig:access-set}
  \vspace{-1em}
\end{figure*}

MergePipe realizes the access-mask abstraction through a catalog--plan--execute loop; detailed algorithms are in Appendix~\ref{app:algorithms}. 
Given a base checkpoint, expert pool, merge operator, and expert-I/O budget, it returns a logical merged checkpoint and a replayable manifest recording the access mask, plan hash, touched blocks, realized expert reads, and lineage. 
Figure~\ref{fig:access-set} illustrates how budgeted access sets are planned and executed in weight space.
More details of the system are shown in Figure~\ref{fig:system-overview-appendix}.

\textbf{Catalog.}
MergePipe exposes checkpoints as block-structured weight data rather than opaque files. 
For each tensor block, the catalog stores size, layout, hashes, and lightweight statistics such as sketches or coverage hints. 
These reusable metadata allow cost estimation and access planning without repeatedly scanning all expert checkpoints.

\textbf{Planner.}
The planner constructs an access mask $A$ satisfying $C_{\mathrm{expert}}(A)\le B$. 
It ranks candidate expert deltas using catalog statistics and selects blocks under the expert-I/O budget, with deterministic tensor-level fallback when block metadata are missing. 
The planner does not introduce a new merge rule: it decides which deltas are physically materialized, while the requested operator is applied through its mask-aware instantiation $\Psi_{\mathrm{op}}$, recovering the standard full-read operator at $A=\mathbf{1}$.

\textbf{Executor.}
The executor streams base blocks in checkpoint order and uses \texttt{DeltaIterator} to materialize only selected expert deltas from full checkpoints, explicit deltas, or LoRA-style adapters \citep{hu2022lora}. 
Omitted entries are passed as mask values rather than storage reads. 
For each block, MergePipe applies $\Psi_{\mathrm{op}}$ and records touched blocks, contributing experts, and realized I/O in the manifest. 
Our comparisons use full logical checkpoint materialization and isolate expert-read I/O before optional overlay optimization.

\begin{algorithm}[t]
\caption{\textsc{PlanGen}: greedy budget-aware plan}
\label{alg:plangen-main}
\begin{algorithmic}[1]
\REQUIRE {\small Experts $\{M_i\}_{i=1}^{K}$, catalog $\mathcal{C}$, operator $\mathrm{op}$, budget $B$}
\ENSURE Budget-feasible merge plan $\pi$
\STATE Build candidate expert-block set $\mathcal{Q}=\{(i,t,b)\}$ from block metadata.
\STATE Score candidates $s(i,t,b)$ using norms, sketches, coverage, or deterministic fallback metadata.
\STATE Sort $\mathcal{Q}$ by decreasing score with stable tensor/block tie-breaking.
\STATE $\mathcal{R}_{\pi}\leftarrow\emptyset$, $\widehat{C}\leftarrow0$.
\FOR{candidate $(i,t,b)\in\mathcal{Q}$}
  \IF{$\widehat{C}+c_{i,t,b}\le B$}
    \STATE $\mathcal{R}_{\pi}\leftarrow\mathcal{R}_{\pi}\cup\{(i,t,b)\}$; $\widehat{C}\leftarrow \widehat{C}+c_{i,t,b}$.
  \ENDIF
\ENDFOR
\STATE Record operator parameters, traversal order, selected-block digest, and $\widehat{C}$.
\STATE \textbf{return} $\pi=(\mathrm{op},\theta,\mathcal{R}_{\pi},\mathrm{order})$.
\end{algorithmic}
\end{algorithm}

\begin{algorithm}[t]
\caption{\textsc{ExecuteMerge}: \\ budget-enforced streaming execution}
\label{alg:execute-main}
\small
\begin{algorithmic}[1]
\REQUIRE Plan $\pi$, base $M_0$, experts $\{M_i\}$, storage $\mathcal{S}$, catalog $\mathcal{C}$, transaction manager $\mathcal{T}$
\ENSURE Snapshot id $sid$ and manifest $\mathsf{man}$
\STATE $\mathcal{T}.\textsc{Begin}()$; open staging writer $w$; initialize touch and coverage maps.
\FOR{tensor $t$ in $\pi.\textsc{TensorOrder}()$}
  \STATE Initialize \texttt{DeltaIterator} $D$ for $t$.
  \FOR{block $b$ in $\pi.\textsc{BlocksToMaterialize}(t)$}
    \STATE Read base block $x_0\leftarrow\mathcal{S}.\textsc{ReadBaseBlock}(M_0,t,b)$.
    \STATE $\eta\leftarrow D.\textsc{PullMasked}(b,\mathcal{R}_{\pi})$.
    \STATE $x\leftarrow\textsc{ApplyBudgetedOp}(x_0,\eta,\pi.\textsc{Op}(t,b))$.
    \STATE $w.\textsc{WriteBlockOrReference}(t,b,x,M_0)$; update touch and coverage.
  \ENDFOR
\ENDFOR
\STATE Validate hashes; build $\mathsf{man}\leftarrow\mathcal{C}.\textsc{BuildManifest}(\pi,\textit{touch},\textit{coverage})$.
\STATE $sid\leftarrow\mathcal{T}.\textsc{AtomicPublish}(w,\mathsf{man})$; $\mathcal{C}.\textsc{CommitRecord}(sid,\mathsf{man})$.
\STATE $\mathcal{T}.\textsc{Commit}()$; \textbf{return} $(sid,\mathsf{man})$.
\end{algorithmic}
\end{algorithm}

\vspace{-0.9em}
\section{Experiments}
\label{sec:experiments}

We evaluate the causal chain behind MergePipe: full-read expert access grows with $K$; budgeted access caps this growth; wall time follows expert-read I/O; and the resulting approximation remains useful downstream. Experiments cover Qwen3-0.6B/1.7B/8B, Llama-3.2-3B, and Llama-3.1-8B, with up to 20 experts for Qwen/Llama-8B and 25 for Llama-3.2-3B. Merges are CPU-only, use SSD-backed storage, and disable OS-level file caching unless stated otherwise. Baselines implement the same operators but scan all required expert checkpoints on every invocation. Appendix~\ref{app:additional} reports the budget sweep, I/O and overhead breakdowns, a compact operator table, and the full quality table.

\begin{figure*}[t]
  \centering
  \includegraphics[width=0.93\textwidth]{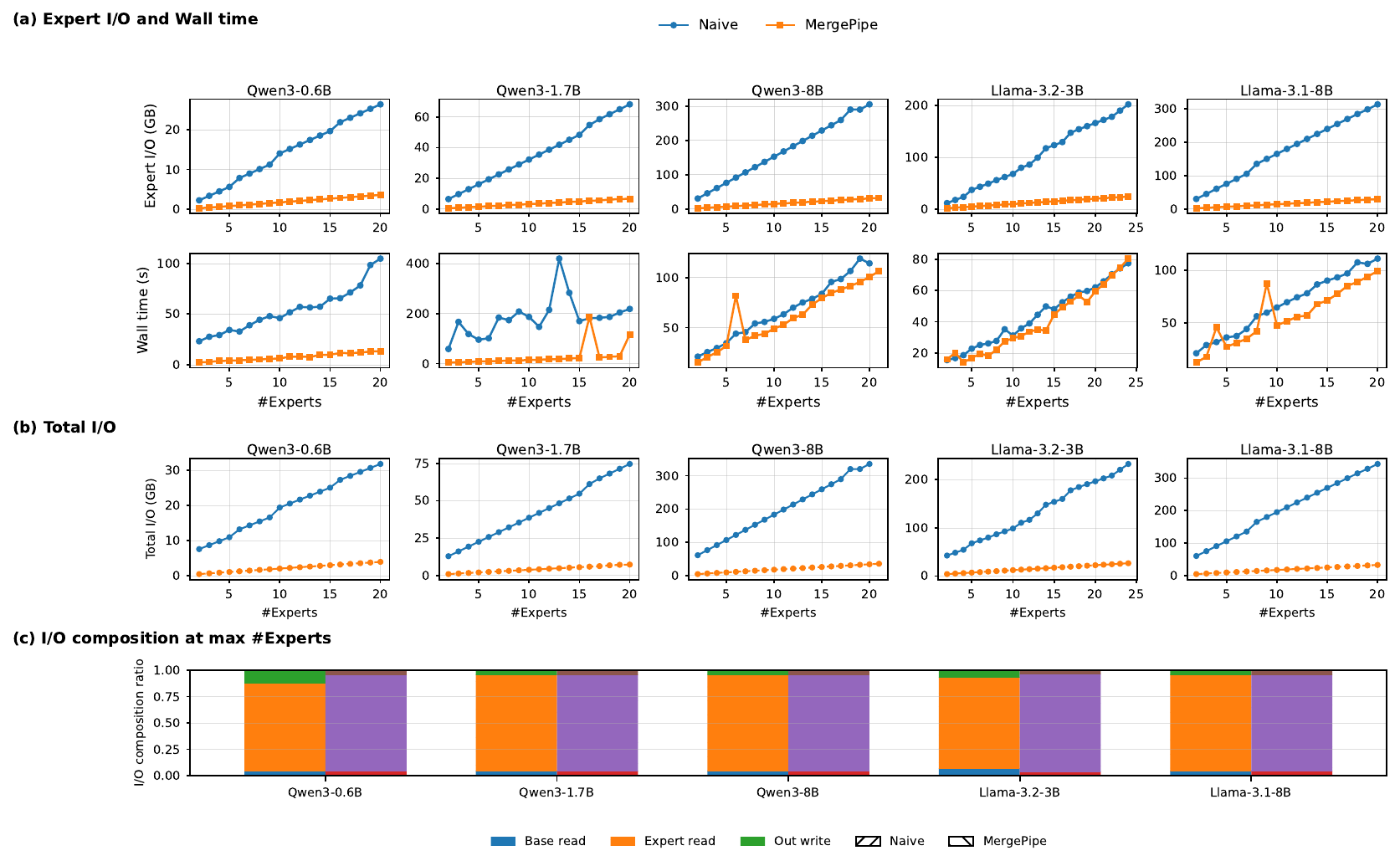}
  \caption{\textbf{Scaling with the number of experts.}
  Full-read merging repeatedly scans expert checkpoints, so expert-read I/O and wall time grow with $K$.
  MergePipe enforces a fixed expert-I/O budget, keeping expert reads nearly flat and shifting the remaining cost toward the unavoidable checkpoint boundary.}
  \label{fig:scaling}
  \vspace{-1.5em}
\end{figure*}

\vspace{-0.3em}
\textbf{Scaling and budget behavior.}
\Cref{fig:scaling} shows that naive merging has near-linear expert-I/O growth and matching wall-time growth as $K$ increases. MergePipe keeps the access set within $B$, making expert reads nearly flat and yielding order-of-magnitude expert-I/O reductions and up to $11\times$ speedups. Budget sweeps from 10\% to 100\% further show monotone realized expert-read I/O, accessed-block ratio, and wall-time growth (Appendix~\ref{app:budget}), making $B$ a direct throughput--fidelity knob.

\begin{figure}[t]
  \centering
  \includegraphics[width=0.99\linewidth]{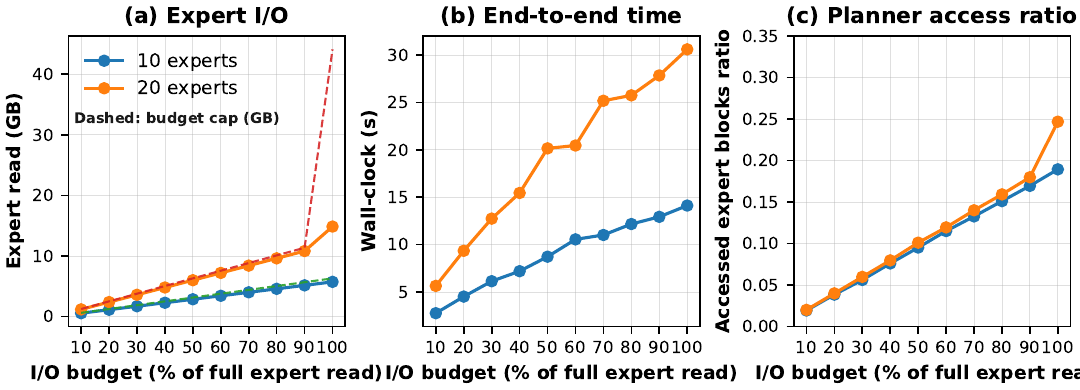}
  \caption{\textbf{Budget-aware planning behavior.}
  \textbf{(a)} Realized expert reads grow monotonically with the requested I/O budget and remain under the cap.
  \textbf{(b)} End-to-end wall time follows expert-read volume.
  \textbf{(c)} The fraction of accessed expert blocks expands smoothly as more budget is allocated.}
  \label{fig:budget-main}
  \vspace{-2.3em}
\end{figure}

\vspace{-0.3em}
\textbf{Operator generality.}
On Llama-3.1-8B, MergePipe reduces expert-read I/O across AVG, TIES, and DARE because access planning precedes local merge semantics. 
Sparse methods benefit most: at $K{=}8$, TIES I/O drops from 79.5GB to 3.46GB and wall time from 614s to 51s; at $K{=}20$, TIES still reads 3.46GB while the naive pipeline exceeds 174GB, with a 70.4\% wall-time reduction.

\begin{table}[t]
\centering
\small
\setlength{\tabcolsep}{3.5pt}
\caption{\textbf{Operator generality on Llama-3.1-8B.}
MergePipe applies the same access-budget abstraction to AVG, TIES, and DARE. I/O is reported in GB.}
\label{tab:operator-main}
\begin{tabular}{lccccc}
\toprule
Op. & $K$ & Naive I/O & MP I/O & Naive s / MP s & Reduc. \\
\midrule
AVG  & 8  & 70.7 & 3.46 & 136.2 / 64.1  & 53.0\% \\
AVG  & 20 & 171.0 & 3.46 & 363.0 / 338.4 & 6.8\% \\
TIES & 8  & 79.5 & 3.46 & 614.3 / 51.1  & 91.7\% \\
TIES & 20 & 174.4 & 3.46 & 1203.7 / 356.3 & 70.4\% \\
DARE & 8  & 79.5 & 3.46 & 99.1 / 44.7   & 54.9\% \\
DARE & 20 & 173.2 & 3.46 & 464.2 / 414.4 & 10.7\% \\
\bottomrule
\end{tabular}
\end{table}

\begin{table}[t]
\centering
\small
\setlength{\tabcolsep}{4.5pt}
\caption{\textbf{Fidelity under budgeted expert access.}
Qwen3-0.6B, TIES, $K{=}20$.
Budget is normalized to the full-read TIES endpoint; touched ratio is measured after TIES sparsification.}
\label{tab:fidelity}
\vspace{-0.3em}
\resizebox{\linewidth}{!}{
\begin{tabular}{cccccc}
\toprule
Budget & Touched & Rel. $\ell_2$ & HumanEval & IFEval & DROP \\
\midrule
1.0 & 0.247 & 0 & 39.71 & 68.56 & 66.81 \\
0.9 & 0.180 & $7.23{\times}10^{-4}$ & 42.68 & 68.82 & 67.23 \\
0.7 & 0.140 & $8.30{\times}10^{-4}$ & 37.80 & 68.47 & 66.26 \\
0.5 & 0.101 & $8.84{\times}10^{-4}$ & 40.24 & 68.11 & 66.00 \\
\bottomrule
\end{tabular}}
\vspace{-1.7em}
\end{table}

\vspace{-0.3em}
\textbf{Quality under bounded access.}
Budgeted access is a controlled approximation, so we measure both parameter deviation and downstream quality. \Cref{tab:fidelity} compares budgeted TIES outputs with the full-read output on Qwen3-0.6B with 20 experts. The full-read TIES endpoint has touched ratio below one because TIES itself sparsifies updates; the budget is normalized to this operator-induced access cost. Even at 0.5 budget, relative $\ell_2$ deviation remains $O(10^{-3})$. HumanEval \citep{chen2021evaluating}, IFEval \citep{zhou2023instruction}, and DROP \citep{dua2019drop} stay close to the full-read baseline and show no monotonic degradation, indicating a favorable speed--fidelity frontier in this setting.

\begin{table}[t]
\centering
\small
\setlength{\tabcolsep}{3.2pt}
% \vspace{-0.5em}
\caption{\textbf{Parameter deviation and downstream quality under different budgets} (Qwen3-0.6B, TIES, $K{=}20$).}
\label{tab:fidelity-full-main}
\resizebox{\linewidth}{!}{
\begin{tabular}{c|ccc|ccc}
\toprule
 & \multicolumn{3}{c|}{\textbf{Deviation from full-read merge}} &
 \multicolumn{3}{c}{\textbf{Downstream score}} \\
\cmidrule(lr){2-4} \cmidrule(lr){5-7}
Budget & Touched & Rel. $\ell_2$ & P95 block & HumanEval & IFEval & DROP \\
\midrule
1.0 & 0.247 & 0 & 0 & 39.71 & 68.56 & 66.81 \\
0.9 & 0.180 & $7.23{\times}10^{-4}$ & $3.66{\times}10^{-3}$ & 42.68 & 68.82 & 67.23 \\
0.8 & 0.159 & $7.77{\times}10^{-4}$ & $3.66{\times}10^{-3}$ & 36.59 & 68.94 & 67.08 \\
0.7 & 0.140 & $8.30{\times}10^{-4}$ & $3.98{\times}10^{-3}$ & 37.80 & 68.47 & 66.26 \\
0.6 & 0.119 & $8.30{\times}10^{-4}$ & $3.98{\times}10^{-3}$ & 37.37 & 67.87 & 66.10 \\
0.5 & 0.101 & $8.84{\times}10^{-4}$ & $3.98{\times}10^{-3}$ & 40.24 & 68.11 & 66.00 \\
\bottomrule
\end{tabular}}
\vspace{-1em}
\end{table}

\vspace{-0.3em}
\textbf{System overhead and scope.}
Planning and metadata are small compared with tensor streaming: in a Qwen3-0.6B, 16-expert run, planning takes 1.21s (about 1\% of execution), the manifest is 812KB, and catalog storage is 3.79\% of total I/O (Appendix~\ref{app:overhead}). MergePipe targets offline, iterative merging with many disk-resident experts; gains naturally shrink for small expert sets, dense full-read regimes, or GPU-resident in-memory fusion.

\begin{table}[t]
\centering
\small
\setlength{\tabcolsep}{4pt}
\caption{Execution and system costs (Qwen3-0.6B, 16 experts).}
\label{tab:execution-overheads-main}
\resizebox{\linewidth}{!}{
\begin{tabular}{lcc}
\toprule
Metric & Absolute value & Relative cost \\
\midrule
Planning time & 1.21s & 1.04\% of execution \\
Execution time & 116.62s & 100\% \\
Executed expert-read I/O & 23,349MB & 1.18$\times$ pre-plan \\
Metadata catalog size & 3,468MB & 3.79\% of total I/O \\
Execution manifest size & 812KB & $<$0.001\% \\
\bottomrule
\end{tabular}}
\vspace{-2em}
\end{table}

\section{Related Work}
\label{sec:related}

\textbf{Model merging.}
Weight-space merging methods combine checkpoints or task vectors into one deployable model. 
Recent surveys summarize the broader model-fusion and model-merging landscape
\citep{yang2024model,lu2024merge,zhou2025democratizing,zhou2026model}.
Dense methods include model soups \citep{wortsman2022model}; sparse or interference-aware methods include TIES, DARE, low-rank variants, and activation- or sensitivity-informed merging \citep{yu2024language,liu2025lore,nobari2025activation,liu2025sens}. 
Recent LLM work further studies post-merge feature calibration and quantization \citep{gu2026featcal,wang2026pmq}, continual post-training conflicts \citep{wang2026geometry}, and domain-specific expert composition in weight space \citep{wang2026discovering}. 
Together, these works study \emph{how} weights should be combined, adapted, or deployed. MergePipe is complementary: it studies \emph{which expert weights must be accessed} when checkpoints and expert pools are large.

\textbf{Machine Learning Management Systems.}
Large-scale LLM development produces many checkpoints, deltas, and merged variants, motivating systems for experiment tracking, artifact logging, versioning, workflow orchestration, and provenance~\citep{zaharia2018accelerating,barreto2024dvc,vadde2024devops,eggers2024automating,bux2015saasfee,green2007provenance,ruan2021lineagechain}.
These systems improve reproducibility and traceability, but they mainly manage checkpoints and pipelines rather than optimizing the parameter I/O pattern of LLM merging.
MergePipe addresses this bottleneck by treating parameters as block-level execution units, caching reusable tensor statistics, and planning merge execution under an explicit expert-read budget.
This enables predictable, budget-aware checkpoint access for large-scale LLM merging, complementing prior ML management systems.

\section{Conclusion}
\label{sec:conclusion}

We presented MergePipe, a budgeted access-set abstraction for scalable weight-space model merging. 
Our central observation is that, for LLM checkpoint families, merging is constrained not only by how expert deltas are combined, but also by which expert deltas are physically read. 
By making expert access a first-class budgeted resource, MergePipe separates the logical merge rule from its physical access pattern, provides budget-sound and full-budget-consistent execution, and exposes a practical speed--fidelity frontier. 
Across Qwen and Llama checkpoint families, this view turns full-read expert scans into bounded expert access, yielding large I/O and runtime reductions while retaining downstream behavior in representative budgeted merges. 
More broadly, our results suggest that as model families continue to grow, weight-space methods should be paired with execution layers that treat model weights as structured, budgeted data rather than opaque checkpoint files.

\vspace{-0.5em}
\section*{Acknowledgements}
\vspace{-0.5em}

This paper is fully supported by a grant from the Research Grants Council of the Hong Kong Special Administrative Region, China (Project No. T41-517/25-N).

\bibliography{ref}
\bibliographystyle{icml2026}

\newpage
\appendix

\section{Formalization and Algorithms}
\label{app:formal}

This section expands the formal model, merge-operator semantics, catalog schema, and execution algorithms used by MergePipe.

\Cref{fig:system-overview-appendix} gives the implementation view behind the access-set abstraction in the main text.
MergePipe indexes LLM checkpoints as block-level weight data, plans budget-feasible expert-delta access, executes the requested mask-aware merge operator, and publishes a logical checkpoint together with a replayable manifest.
This system view is complementary to the main formulation: the central object remains the expert access mask, while the runtime components make that mask executable at checkpoint scale.

\begin{figure*}[t]
  \centering
  \includegraphics[width=0.95\textwidth]{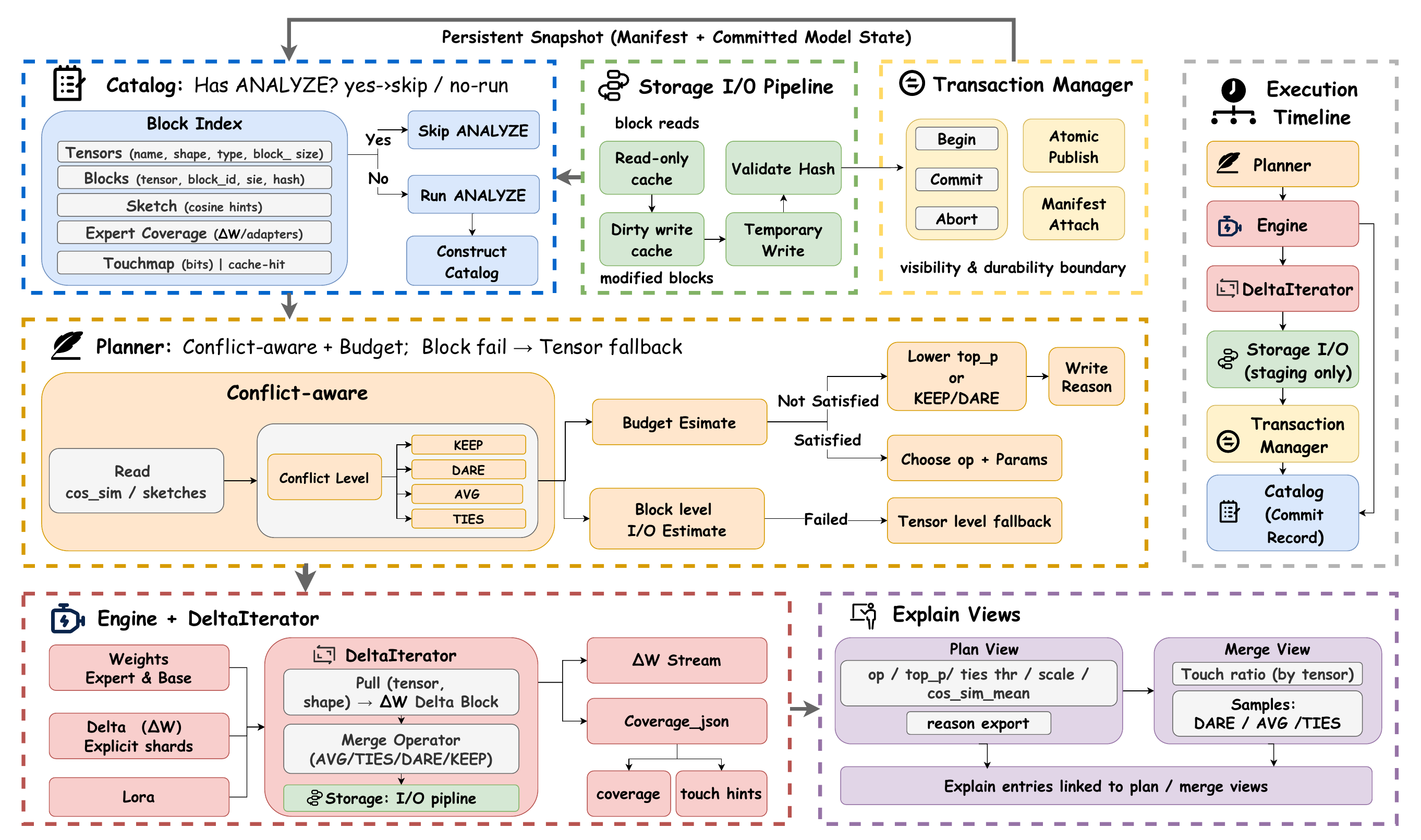}
  \caption{\textbf{MergePipe system overview.}
  The runtime realizes budget-aware weight-space merging through block-level cataloging, access-set planning, mask-aware execution, and manifest-based replay.
  The planner controls expert-delta reads under the I/O budget, while the executor streams only selected expert blocks and materializes the resulting logical checkpoint.}
  \label{fig:system-overview-appendix}
  % \vspace{-0.5em}
\end{figure*}

\subsection{Cost Accounting and Operator Semantics}
\label{app:operators-formal}

The total merge cost decomposes into base reads, expert reads, output writes, and metadata:
\begin{equation}
C_{\mathrm{merge}}=C_{\mathrm{base}}+C_{\mathrm{expert}}+C_{\mathrm{out}}+C_{\mathrm{meta}} .
\end{equation}
MergePipe constrains only the expert-read term. For a sparse plan $\mathcal{R}_{\pi}$, the induced access mask is
\begin{equation}
\begin{aligned}
A_{i,t,b}&=\mathbf{1}\{(i,t,b)\in\mathcal{R}_{\pi}\},
\\
\widehat{C}_{\mathrm{expert}}(\pi)&=\sum_{(i,t,b)\in\mathcal{R}_{\pi}}c_{i,t,b}\le B,
\end{aligned}
\end{equation}
where $c_{i,t,b}$ is the same physical accounting unit used by the executor. In uncompressed block-aligned runs, it equals the stored byte length of block $b$.

For fixed-coefficient additive operators, the mask can be implemented through a zero-completed tuple
\begin{equation}
\bar{\Delta}^{A}_{i,t,b}=A_{i,t,b}\Delta_{i,t,b} .
\end{equation}
The zero entries are logical handles and do not trigger storage reads. Average merging is then
\begin{equation}
\Psi_{\mathrm{AVG}}(A,\Delta)=\sum_{i=1}^{K}\alpha_i\bar{\Delta}^{A}_{i},
\qquad \sum_i\alpha_i=1.
\end{equation}
DARE applies the same idea before random drop/rescale, using a fixed seed or mask $\omega$ when compared to the full-read endpoint. TIES trims, elects signs, and averages sign-consistent entries after masking. Because access masks can alter random drops, top-$k$ sparsification, and sign election, we do not claim a global smooth error bound for TIES/DARE; instead we report parameter deviation and downstream quality empirically.

If an implementation renormalizes an additive merge over only selected experts, the omitted-delta bound in the main text no longer applies. Let the full-read additive merge use coefficients $\alpha_{i,t,b}$, and let the selected-only budgeted merge use coefficients $\beta_{i,t,b}(A)$ with $\beta_{i,t,b}(A)=0$ whenever $A_{i,t,b}=0$. Then
\begin{equation}
M_{\mathrm{full}}[t,b]-M_A[t,b]
=
\sum_i
(\alpha_{i,t,b}-\beta_{i,t,b}(A))\Delta_{i,t,b},
\end{equation}
and, defining
\begin{equation}
r_{t,b}(A)=\sum_i |\alpha_{i,t,b}-\beta_{i,t,b}(A)|\|\Delta_{i,t,b}\|_2,
\end{equation}
we have
\begin{equation}
\|M_{\mathrm{full}}-M_A\|_2
\le
\left[\sum_{t,b}r_{t,b}(A)^2\right]^{1/2}.
\end{equation}
This coefficient-drift form is the correct bound for selected-only averaging.

\subsection{Planning and Execution Sketches}
\label{app:algorithms}

The planning and execution sketches are included in the main text as \Cref{alg:plangen-main,alg:execute-main}.

\subsection{Proofs}
\label{app:proofs}

\begin{proof}[Proof of \Cref{prop:budget-soundness}]
For every expert block $(i,t,b)$, assumption (ii) gives $N^{\mathrm{run}}_{i,t,b}\le A_{i,t,b}$. Since $c_{i,t,b}\ge0$,
\begin{equation}
\begin{aligned}
C_{\mathrm{expert}}^{\mathrm{run}}(A)
&=\sum_{i,t,b}N^{\mathrm{run}}_{i,t,b}c_{i,t,b}
\le\sum_{i,t,b}A_{i,t,b}c_{i,t,b}
\\
&=C_{\mathrm{expert}}(A)
\le B.
\end{aligned}
\end{equation}
This proves budget soundness.
For full-budget consistency, suppose $B\ge C_{\mathrm{expert}}(\mathbf{1})$ and the planner selects $A=\mathbf{1}$. By \Cref{eq:psi-consistency}, for every block $(t,b)$,
\begin{align*}
M_A[t,b]
&=M_0[t,b]+\Psi_{\mathrm{op},t,b}(\mathbf{1},\Delta_{\cdot,t,b};\theta,\omega)\\
&=M_0[t,b]+\Phi_{\mathrm{op},t,b}(\Delta_{\cdot,t,b};\theta,\omega)
=M_{\mathrm{full}}[t,b].
\end{align*}
Thus $M_A=M_{\mathrm{full}}$ blockwise. For randomized operators, the equality is pathwise under the same seed or mask $\omega$; without fixing $\omega$, it is equality in distribution.
\end{proof}

\begin{proof}[Proof of \Cref{cor:expert-scaling}]
By \Cref{prop:budget-soundness}, $C_{\mathrm{expert}}^{\mathrm{run}}(A)\le B$. Full-read execution reads every expert and incurs $C_{\mathrm{expert}}^{\mathrm{full}}(K)=K\bar{C}$. Dividing by $K\bar{C}$ gives the claim.
\end{proof}

\begin{proof}[Proof of \Cref{prop:additive-bound}]
For the full-read and budgeted additive merges,
\begin{align*}
M_{\mathrm{full}}[t,b]
&=M_0[t,b]+\sum_{i=1}^{K}\alpha_{i,t,b}\Delta_{i,t,b},\\
M_A[t,b]
&=M_0[t,b]+\sum_{i=1}^{K}\alpha_{i,t,b}A_{i,t,b}\Delta_{i,t,b}.
\end{align*}
Subtracting gives
\[
M_{\mathrm{full}}[t,b]-M_A[t,b]
=\sum_{i=1}^{K}(1-A_{i,t,b})\alpha_{i,t,b}\Delta_{i,t,b}.
\]
Let $d_{t,b}=M_{\mathrm{full}}[t,b]-M_A[t,b]$. Since blocks are disjoint parameter coordinates,
\[
\|M_{\mathrm{full}}-M_A\|_2=\left(\sum_{t,b}\|d_{t,b}\|_2^2\right)^{1/2}.
\]
The triangle inequality gives $\|d_{t,b}\|_2\le q_{t,b}(A)$. Substitution yields the first inequality in \Cref{eq:omission-bound}; the second follows from $(\sum z_{t,b}^2)^{1/2}\le\sum z_{t,b}$ for nonnegative $z_{t,b}$.
\end{proof}

Atomic visibility follows from \Cref{alg:execute-main}: a run either publishes one snapshot/manifest pair or leaves no externally visible state.

\section{Additional Experimental Evidence}
\label{app:additional}

This supplement preserves the non-duplicated additional evidence most directly used by the main text: budget-controlled planning behavior and I/O/overhead decomposition. Baselines implement the same merge operator as MergePipe, use full-read expert access, and run with OS-level file caching disabled.

\textbf{Budget-Controlled Planning.}
\label{app:budget}
The budget-controlled planning results are included in the main text as \Cref{fig:budget-main}.

\textbf{I/O Breakdown and Overhead.}
\label{app:overhead}
\Cref{fig:io-overhead-appendix} shows that the gains come from reducing the expert-read term rather than from metadata effects. Base reads and output writes are checkpoint-boundary costs, while planning and transactional overhead remain small compared with tensor streaming.

\begin{figure}[t]
  \centering
  \begin{minipage}[t]{0.48\linewidth}
    \centering
    \includegraphics[width=\linewidth]{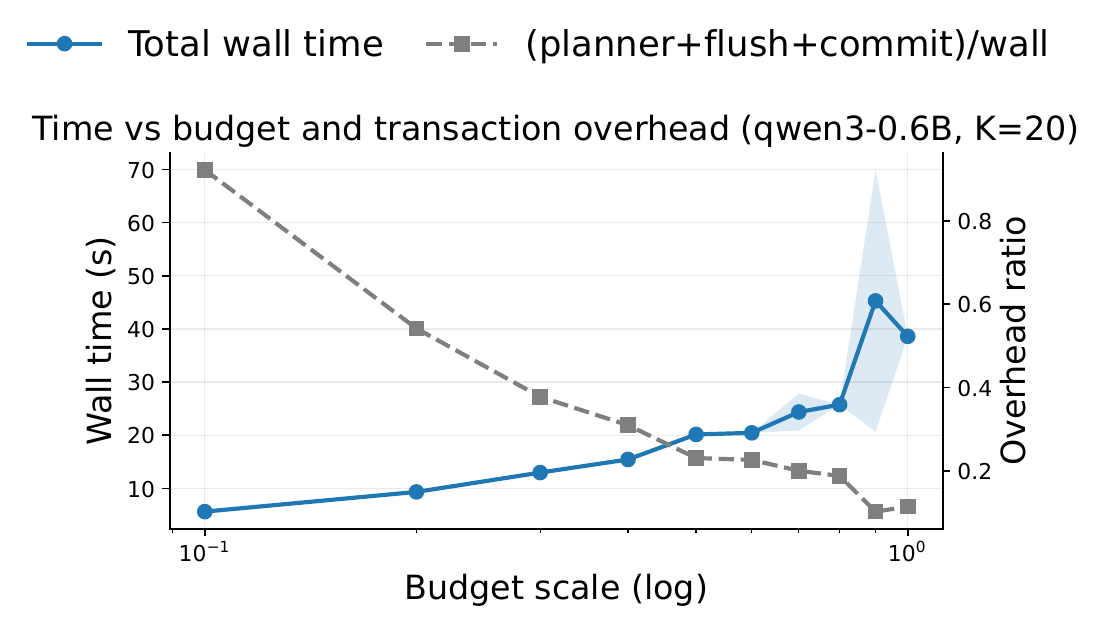}
  \end{minipage}
  \begin{minipage}[t]{0.48\linewidth}
    \centering
    \includegraphics[width=\linewidth]{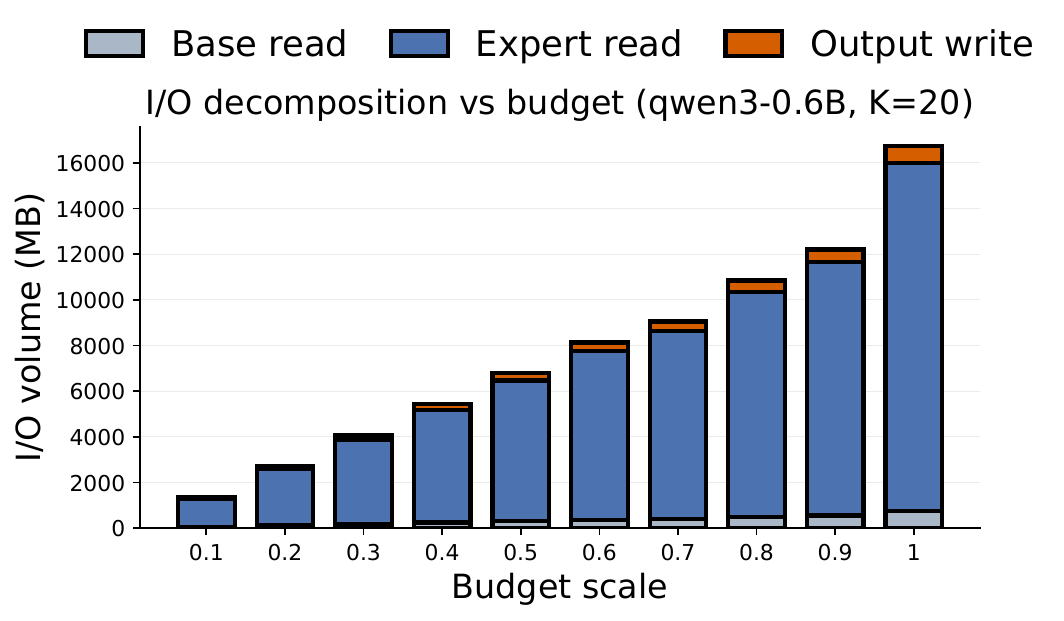}
  \end{minipage}
  \\
  \begin{minipage}[t]{0.95\linewidth}
    \centering
    \includegraphics[width=\linewidth]{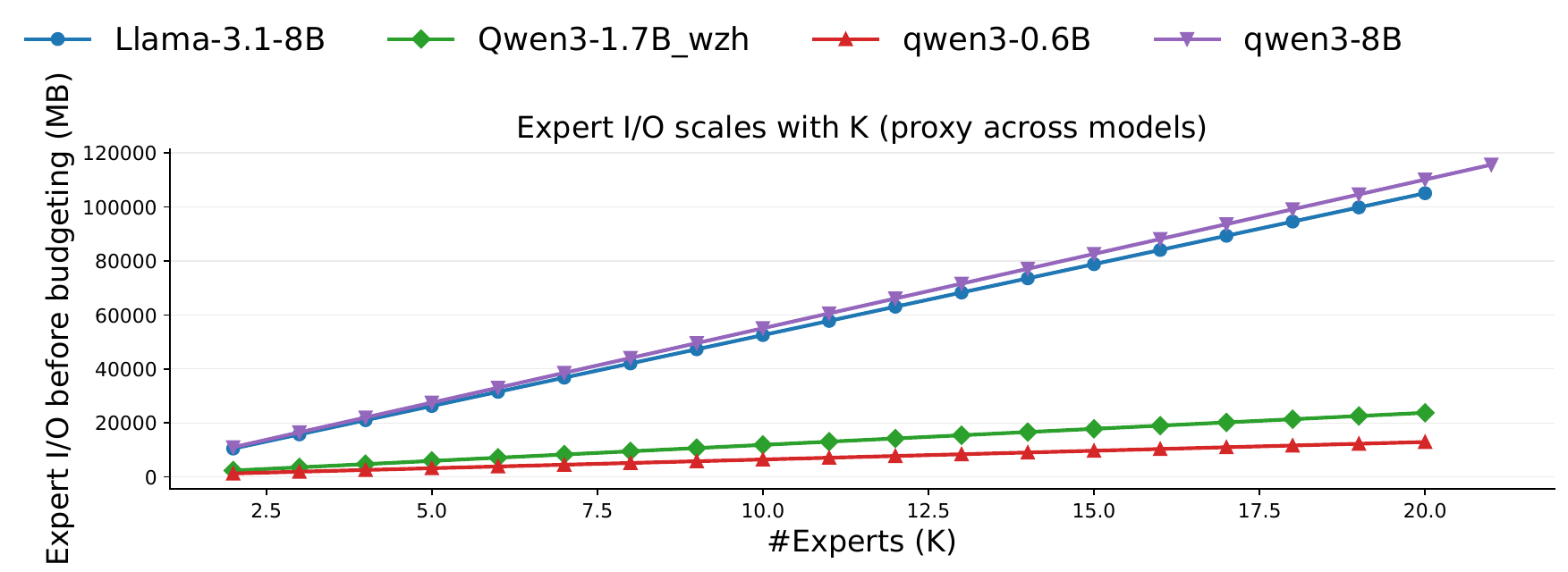}
  \end{minipage}
  \caption{\textbf{Where MergePipe saves time.}
  \textbf{Top-left:} planning, flush, and commit are small relative to execution.
  \textbf{Top-right:} tightening the budget primarily removes expert reads, while base reads and output writes remain nearly fixed.
  \textbf{Bottom:} before budgeting, expert reads scale with the number of experts.}
  \label{fig:io-overhead-appendix}
\end{figure}

\section{Limitations}
\label{app:limitations}

MergePipe targets budgeted weight-space access for checkpoint merging and is complementary to behavior-level fusion methods based on preference optimization, on-policy distillation, or logit-space alignment~\citep{gu2025infifpo,wang2026disagreementlearnabletokenteachability,wang2026infigfusion}.
It assumes experts share a common weight coordinate system and does not address permutation, symmetry, or representation alignment.
Budgeted merging is approximate: only the full-budget setting recovers the standard full-read merge, while lower budgets rely on mask-aware execution.
Its quality is therefore empirical, especially for nonlinear sparse operators such as TIES and DARE, and the benefits are largest when disk-resident expert-read I/O dominates.

\end{document}